%% file: arxiv.tex
\documentclass{article}
\usepackage[numbers]{natbib}

\usepackage{fullpage}

\usepackage{microtype}
\usepackage{graphicx}
\usepackage{subfigure}
\usepackage{algorithm}
\usepackage{algorithmic}
\usepackage{booktabs} 

\usepackage[colorlinks=true, linkcolor=blue, citecolor=blue, breaklinks=true]{hyperref}

\usepackage{amsmath,amsfonts}
\usepackage{amsthm}
\usepackage{amssymb}
\usepackage{bm}
\usepackage{enumerate} 
\usepackage{url} 
\usepackage{natbib}
\usepackage{makecell}

\usepackage[]{color-edits}
\addauthor{HL}{red}
\addauthor{D}{blue}

\newcommand{\TODO}[1]{%
\ifmmode
\text{\textcolor{red}{TODO: #1}}
\else
\textcolor{red}{TODO: #1}
\fi
}

\newcommand{\todo}[1]{%
\ifmmode
\text{\textcolor{red}{TODO: #1}}
\else
\textcolor{red}{TODO: #1}
\fi
}

\DeclareMathOperator{\E}{\mathbb E}
\DeclareBoldMathCommand{\w}{w}
\DeclareBoldMathCommand{\s}{s}
\DeclareBoldMathCommand{\x}{x}
\DeclareBoldMathCommand{\1}{1}
\DeclareBoldMathCommand{\G}{G}
\DeclareBoldMathCommand{\X}{X}
\DeclareBoldMathCommand{\A}{A}
\DeclareBoldMathCommand{\a}{a}
\DeclareBoldMathCommand{\e}{e}
\DeclareBoldMathCommand{\b}{b}
\DeclareBoldMathCommand{\u}{u}
\DeclareBoldMathCommand{\g}{g}
\DeclareBoldMathCommand{\z}{z}
\DeclareBoldMathCommand{\y}{y}
\DeclareBoldMathCommand{\h}{h}
\DeclareBoldMathCommand{\I}{I}
\DeclareBoldMathCommand{\one}{1}
\DeclareBoldMathCommand{\l}{l}
\DeclareBoldMathCommand{\0}{0}
\DeclareBoldMathCommand{\t}{t}
\DeclareBoldMathCommand{\v}{v}
\DeclareBoldMathCommand{\r}{r}
\DeclareBoldMathCommand{\s}{s}
\DeclareBoldMathCommand{\q}{q}
\DeclareBoldMathCommand{\h}{h}
\DeclareBoldMathCommand{\bth}{\theta}

\newcommand{\regret}{\mathcal{R}}

\DeclareBoldMathCommand{\S}{S}
\DeclareBoldMathCommand{\c}{c}
\DeclareBoldMathCommand{\np}{\theta}
\DeclareBoldMathCommand{\Mu}{\mu}
\DeclareBoldMathCommand{\TAU}{\tau}
\DeclareBoldMathCommand{\covmat}{\Sigma}

\newcommand{\half}{\tfrac{1}{2}}

\newcommand{\domainw}{\mathcal{W}}

\newcommand{\reals}{\mathbb{R}}
\newcommand{\ball}{\mathbb{B}}
\newcommand{\sphere}{\mathbb{S}}

\newcommand{\sumT}{\sum_{t=1}^T}

\newcommand{\ghatt}{\hat{\g}_t}

\newcommand{\blosst}{\bar{\ell}_t}

\newcommand{\vhat}{\hat{v}}
\newcommand{\gt}{\g_t}

\newcommand{\wtilt}{\tilde{\w}_t}

\newcommand{\algv}{\mathcal{A}_\mathcal{V}}
\newcommand{\algz}{\mathcal{A}_\mathcal{Z}}

\newcommand\inner[2]{\langle #1, #2 \rangle}

\newcommand\doma[1]{\mathcal{#1}}
\newcommand{\otil}{\widetilde{O}}
\newcommand{\ud}{\u}
\newcommand{\tud}{\tilde{\u}}
\newcommand{\edt}{\ell_t}
\newcommand{\edvt}{\ell^{v_t}_t}

\DeclareBoldMathCommand{\Loss}{\ell}

\DeclareMathOperator*{\argmin}{arg\,min}

\newenvironment{interface}[1][htb]
  {\floatname{algorithm}{Interface}\begin{algorithm}[#1]}{\end{algorithm}\floatname{algorithm}{Algorithm}}

\newtheorem{theorem}{Theorem}
\newtheorem{lemma}{Lemma}

\theoremstyle{definition}

\title{Comparator-Adaptive Convex Bandits}

\author{Dirk van der Hoeven \\Leiden University \\\texttt{dirk@dirkvanderhoeven.com} \\
    \and
    Ashok Cutkosky \\
    Boston University \\
    \texttt{ashok@cutkosky.com}\\
    \and 
    Haipeng Luo \\University of Southern California \\
    \texttt{haipengl@usc.edu}}

\begin{document}
\maketitle

\begin{abstract}
We study bandit convex optimization methods that adapt to the norm of the comparator, a topic that has only been studied before for its full-information counterpart. Specifically, we develop convex bandit algorithms with regret bounds that are small whenever the norm of the comparator is small. 
We first use techniques from the full-information setting to develop comparator-adaptive algorithms for linear bandits. Then, we extend the ideas to convex bandits with Lipschitz or smooth loss functions, using a new single-point gradient estimator and carefully designed surrogate losses.

\end{abstract}

\input{intro}

\input{prelim}

\input{linear}

\input{convexdp1}

\input{conclusion}

\bibliography{myBib}

\clearpage
\appendix

\input{appendixdp1}

\end{document}

%% file: intro.tex

\section{Introduction}\label{sec:introduction}

In many situations, information is readily available. For example, if a gambler were to bet on the outcome of a football game, he can observe the outcome of the game regardless of what bet he made. In other situations, information is scarce. For example, the gambler could be deciding what to eat for dinner: should I eat a salad, a pizza, a sandwich, or not at all? These actions will result in different and unknown outcomes, but the gambler will only see the outcome of the action he actually takes, with one notable exception: not eating result in a predetermined outcome of being very hungry.

These two situation are instantiations of two different settings in online convex optimization: the full information setting and the bandit setting. More formally, both settings are sequential decision making problems where in each round $t = 1, \ldots, T$, a learner has to make a prediction $\w_t \in \domainw\subseteq \reals^d$ and an adversary provides a convex loss function $\ell_t: \domainw \rightarrow \reals$. Afterwards, in the full information setting \citep{zinkevich2003} the learner has access to the loss function $\ell_t$, while in the bandit setting \citep{kleinberg2005nearly, flaxman2005online} the learner only receives the loss evaluated at the prediction, that is, $\ell_t(\w_t)$. In both settings the goal is to minimize the regret with respect to some benchmark point $\u$ in hindsight, referred to as the {\it comparator}. More specifically, the regret against $\u$ is the difference between the total loss incurred by the predictions of the learner and that of the comparator:
\begin{equation*}
\regret_T(\u) = \sumT \ell_t(\w_t) - \ell_t(\u).
\end{equation*} 
When the learner's strategy is randomized, we measure the performance by the expected regret $\E\left[\regret_{T}(\u)\right]$.

\renewcommand{\arraystretch}{1.5}
\begin{table*}[t]
   \centering
   \caption{Summary of main results. Regret is measured with respect to the total loss of an arbitrary point $\u \in \reals^d$ in the unconstrained setting, or an arbitrary point $\u \in \domainw$ in the constrained setting with a decision space $\domainw$ contained in the unit ball. $T$ is the total number of rounds, $1/c$ is radius of the largest ball contained by $\domainw$, and $\nu$ is the self-concordant parameter. Both $c$ and $\nu$ are bounded by $O(d)$.}
   \label{tab:results}
   \resizebox{\textwidth}{!}{\begin{tabular}{|c|c|c|}
   \hline
   Loss functions ($L$-Lipschitz)  & Regret for unconstrained settings & Regret for constrained settings \\
   \hline
   \makecell{Linear  (Section~\ref{sec:conlinban})}& $\otil\left(\|\u\|dL\sqrt{T}\right)$ & $\otil\left(\|\u\|cdL\sqrt{T}\right)$ \\
   \hline
   \makecell{Convex  (Section~\ref{sec:unconstrained convex bandits} and \ref{sec:con convex bandits})}& $\otil\left(\|\u\|L\sqrt{d}T^\frac{3}{4}\right)$ & $\otil\left(\|\u\|cL\sqrt{d}T^\frac{3}{4}\right)$ \\
   \hline
   \makecell{Convex and $\beta$-smooth  (Section~\ref{sec:con convex bandits})}& $\otil\left(\max\{\|\u\|^2, \|\u\|\} \beta(dLT)^\frac{2}{3}\right)$ & - \\ 
   \hline
   \end{tabular}}
\end{table*}   

Standard algorithms in both the full information setting and the bandit setting assume that the learner's decision space $\domainw$ is a convex {\it compact} set and achieve sublinear regret against the optimal comparator in this set: $\u = \argmin_{\u^* \in \domainw}\sumT \ell_t(\u^*)$. To tune these standard algorithms optimally, however, one requires knowledge of the norm of the comparator $\|\u\|$, which is unknown. A common work-around is to simply tune the algorithms in terms of the worst-case norm: $\max_{\u \in \domainw} \|\u\|$, assumed to be $1$ without loss of generality. This results in worst-case bounds that do not take advantage of the case when $\|\u\|$ is small.  
For example, when the loss functions are $L$-Lipschitz, classic Online Gradient Descent~\citep{zinkevich2003} guarantees $\regret_T(\u) = O(L\sqrt{T})$ in the full information setting, while the algorithm of~\citep{flaxman2005online} guarantees
$
\E\left[\regret_T(\u)\right] = O(d\sqrt{L}T^{3/4})
$
in the bandit setting, both of which are independent of $\|\u\|$.

Recently, there has been a series of works in the full information setting that addresses this problem by developing {\it comparator-adaptive} algorithms, whose regret against $\u$ depends on $\|\u\|$ for {\it all $\u \in \domainw$ simultaneously} (see for example \citet{mcmahan2014unconstrained, orabona2016coin, foster2017parameter, cutkosky2017online, kotlowski2017scale, cutkosky2018black, foster2018online, jun2019parameter, van2019user}).
These bounds are often never worse than the standard worst-case bounds, but could be much smaller in the case when there exists a comparator with small norm and reasonably small total loss.
Moreover, most of these results also hold for the so-called {\it unconstrained} setting where $\domainw = \reals^d$, that is, both the learner's predictions and the comparator can be any point in $\reals^d$.
For example, \citet{cutkosky2018black} achieve $\regret_T(\u) = \otil(\|\u\|L\sqrt{T})$ for all $\u$, in both the constrained and unconstrained settings, under full information feedback.\footnote{%
Throughout the paper, the notation $\otil$ hides logarithmic dependence on parameters $T$ and $L$.
}
 
While developing comparator-adaptive algorithms is relatively well-understood at this point in the full information setting, to the best of our knowledge, this has not been studied at all for the more challenging bandit setting.
In this work, we take the first attempt in this direction and develop comparator-adaptive algorithms for several situations, including learning with linear losses, general convex losses, and convex and smooth losses, for both the constrained and unconstrained settings.
Our results are summarized in Table~\ref{tab:results}.
Ignoring other parameters for simplicity, for the linear case, we achieve $\otil(\|\u\|\sqrt{T})$ regret (Section~\ref{sec:conlinban});
for the general convex case, we achieve $\otil(\|\u\|T^\frac{3}{4})$ regret in both the constrained and unconstrained setting (Sections~\ref{sec:unconstrained convex bandits} and \ref{sec:con convex bandits});
and for the convex and smooth case, we achieve $\otil\left(\max\{\|\u\|^2, \|\u\|\} \beta(dLT)^\frac{2}{3}\right)$ regret in the unconstrained setting (Section~\ref{sec:unconstrained convex bandits}).
 


In order to achieve our results for the convex case, we require an assumption on the loss, namely that the value of $\ell_t(\0)$ is known for all $t$.\footnote{For the linear case, this clearly holds since $\ell_t(\0)=0$.} 
While restrictive at first sight, we believe that there are abundant applications where this assumption holds. 
As one instance, in control or reinforcement learning problems, $\0$ may represent some nominal action which has a known outcome: not eating results in hunger, or buying zero inventory will result in zero revenue. Another application is a classification problem where the features are not revealed to the learner. For example, end-users of a prediction service may not feel comfortable revealing their information to the service. 
Instead, they may be willing to do some local computation and report the loss of the service's model. Most classification models (e.g. logistic regression) have the property that the loss of the $\0$ parameter is a known constant regardless of the data, and so this situation would also fit into our framework. Common loss functions that satisfy this assumption are linear loss, logistic loss, and hinge loss.

\paragraph{Techniques}
Our algorithms are based on sophisticated extensions of the black-box reduction introduced by~\citet{cutkosky2018black}, which separately learns the magnitude and the direction of the prediction.
To make the reduction work in the bandit setting, however, new ideas are required, including designing an appropriate surrogate loss function and a new one-point gradient estimator with time-varying parameters.
Note that~\citep{cutkosky2018black} also proposes a method to convert any unconstrained algorithm to a constrained one in the full information setting, but this does not work in the bandit setting for technical reasons. 
Instead, we take a different approach by constraining the magnitude of the prediction directly.

\paragraph{Related work} 
As mentioned, there has been a line of recent works on comparator-adaptive algorithms for the full information setting.
Most of them do not transfer to the bandit setting, except for the approach of \citet{cutkosky2018black} from which we draw heavy inspiration. 
To the best of our knowledge, comparator-adaptive bandit algorithms have not been studied before.
Achieving ``adaptivity'' in a broader sense is generally hard for problems with bandit feedback;
see negative results such as~\citep{daniely2015strongly, lattimore2015pareto} as well as recent progress such as~\citep{chen2019new, foster2019model}.

In terms of worst-case (non-adaptive) regret, the seminal work of~\citep{abernethy2008} is the first to achieve $O(\sqrt{T})$ regret for bandit with linear losses,
and~\citep{kleinberg2005nearly, flaxman2005online} are the first to achieve sublinear regret for general convex case.
Over the past decade, the latter result has been improved in many different ways~\citep{agarwal2010optimal, saha2011improved, agarwal2011stochastic, hazan2014bandit}, and regret of order $O(\sqrt{T})$ under no extra assumptions was recently achieved~\citep{bubeck2015bandit, bubeck2016multi, bubeck2017kernel}.
However, these $O(\sqrt{T})$ bounds are achieved by very complicated algorithms that incur a huge dependence on the dimension $d$.
Our algorithms are more aligned with the simpler ones with milder dimension-dependence~\citep{abernethy2008, flaxman2005online, saha2011improved} and achieve the same dependence on $T$ in different cases.
How to achieve comparator-adaptive regret of order $O(\sqrt{T})$ for the general convex case 
is an important future direction.



%% file: prelim.tex

\section{Preliminaries}\label{sec:prelim}

In this section, we describe our notation, state the definitions we use, and introduce the bandit convex optimization setting formally. We also describe the black-box reduction of \cite{cutkosky2018black} we will use throughout the paper.

\paragraph{Notation and definitions} The inner product between vectors $\g \in \reals^d$ and $\w \in \reals^d$ is denoted by $\langle \w, \g \rangle$. $\reals_+$ denotes the set of positive numbers. The Fenchel conjugate $F^\star$ of a convex function $F$ is defined as $F^\star(\w) = \sup_\g \langle \w, \g \rangle - F(\g)$. $\|\cdot\|$ denotes a norm and $\|\g\|_\star = \sup_{\w: \|\w\|\leq 1} \langle \w, \g \rangle$ denotes the dual norm of $\g$. The Bregman divergence associated with convex function $F$ between points $\x$ and $\y$ is denoted by $B_F(\x\|\y) = F(\x) - F(\y) - \inner{\nabla F(\y)}{ \x- \y}$, where $\nabla F(\x)$ denotes the gradient of $F$ evaluated at $\x$. The unit ball equipped with norm $\|\cdot\|$ is denoted by $\doma{B} = \{\w: \|\w\| \leq 1\}$. The unit sphere with norm $\|\cdot\|$ is denoted by $\doma{S} = \{\w: \|\w\| = 1\}$. The unit ball and sphere with norm $\|\cdot\|_2$ are denoted by $\ball$ and $\sphere$ respectively.  $\x \sim U(\doma{Z})$ denotes that $\x$ follows the uniform distribution over $\doma{Z}$. We say a function $f$ is $\beta$-smooth over the set $\domainw$ if the following holds:
\begin{equation*}
f(\y) \leq f(\x) + \inner{\nabla f(\x)}{\y - \x} + \frac{\beta}{2}\|\x - \y\|^2_2, ~~~ \forall \x, \y \in \domainw.
\end{equation*}
We say a function $f$ is $L$-Lipschitz over the set $\domainw$ if the following holds:
\begin{equation*}
|f(\y) - f(\x)| \leq L\| \y - \x\|_2, ~~~  \forall \x, \y \in \domainw.
\end{equation*}
Throughout the paper we will assume that $\beta, L \geq 1$. Also, by mild abuse of notation, we use $\partial f(x)$ to indicate an arbitrary subgradient of a convex function $f$ at $x$.

All of our algorithms are reductions that use prior algorithms in disparate ways to obtain our new results. In order for these reductions to work, we need some assumptions on the base algorithms. We will encapsulate these assumptions in \emph{interfaces} that describe inputs, outputs, and guarantees described by an algorithm rather than its actual operation (see Interfaces~\ref{alg:scalealg} and~\ref{alg:directionlinban} for examples). We can use specific algorithms from the literature to implement these interfaces, but our results depend only on the properties described in the interfaces.

\subsection{Bandit Convex Optimization}
The bandit convex optimization protocol proceeds in rounds $t = 1, \ldots, T$. In each round $t$ 
the learner plays $\w_t \in \domainw \subseteq \reals^d$. Simultaneously, the environment picks an $L$-Lipschitz convex loss function $\ell_t:\domainw \rightarrow \reals$, after which the learner observes $\ell_t(\w_t)$. Importantly, the learner only observes the loss function evaluated at $\w_t$, not the function itself. This forces the learner to play random points and estimate the feedback he wants to use to update $\w_t$. Therefore, in the bandit feedback setting, the goal is to bound the \emph{expected} regret $\E\left[\regret_T(\u)\right]$, where the expectation is with respect to both the learner and the environment. 

We make a distinction between linear bandits, where $\ell_t(\w) = \inner{\w}{\g_t}$, and convex bandits, where $\ell_t$ can be any $L$-Lipschitz convex function. Throughout the paper, if $\domainw \not = \reals^d$ we assume that $\domainw$ is compact, has a non-empty interior, and contains $\0$. Without loss of generality we assume that $\frac{1}{c}\ball \subseteq \domainw\subseteq \ball$ for some $c \geq 1$.
Some of our bounds depend on $c$, which, without loss of generality, is always bounded by $d$, due to a reshaping trick discussed in~\citep{flaxman2005online}.

\subsection{Black-Box Reductions with Full Information}\label{sec:blackboxfull}

\begin{algorithm}[t]
\caption{Black-Box Reduction with Full Information}\label{alg:black box full}
\begin{algorithmic}[1]
\STATE \textbf{Input:} ``Direction'' algorithm $\mathcal{A}_\mathcal{Z}$ and ``scaling'' algorithm $\mathcal{A}_\mathcal{V}$
\FOR{$t = 1 \ldots T$}
\STATE Get $\z_t \in \mathcal{Z}$ from $\mathcal{A}_\mathcal{Z}$
\STATE Get $v_t \in \reals$ from algorithm $\mathcal{A}_\mathcal{V}$
\STATE Play $\w_t = v_t \z_t$, receive $\gt$
\STATE Send $\gt$ to algorithm $\mathcal{A}_\mathcal{Z}$ as the $t$-th loss vector
\STATE Send $\langle \z_t, \gt \rangle$ to algorithm $\mathcal{A}_\mathcal{V}$ as the $t$-th loss value
\ENDFOR
\end{algorithmic}
\end{algorithm}


Our algorithms are based on a black-box reduction from~\citep{cutkosky2018black} for the full information setting (see Algorithm~\ref{alg:black box full}).
The reduction works as follows. In each round $t$ the algorithms plays $\w_t = v_t \z_t$, where $\z_t \in \mathcal{Z}$ for some domain $\mathcal{Z}$, is the prediction of a constrained algorithm $\mathcal{A}_\mathcal{Z}$, and $v_t$ is the prediction of a one-dimensional algorithm $\mathcal{A}_\mathcal{V}$. The goal of $\mathcal{A}_\mathcal{Z}$ is to learn the directions of the comparator while the goal of $\mathcal{A}_\mathcal{V}$ is to learn the norm of the comparator. 
Let $\gt$ be the gradient of $\ell_t$ at $\w_t$, which is known to the algorithm in the full information setting.
We feed $\gt$ as feedback to $\mathcal{A}_\mathcal{Z}$ and $\langle \z_t , \gt \rangle$ as feedback to $\mathcal{A}_\mathcal{V}$. Although the original presentation considers only $\mathcal{Z}=\doma{B}$, we will need to extend the analysis to more general domains.

As outlined by \citet{cutkosky2018black}, the regret of Algorithm \ref{alg:black box full} decomposes into two parts. The first part of the regret is for learning the norm of $\u$, and is controlled by Algorithm $\mathcal{A}_\mathcal{V}$. The second part of the regret is for learning the direction of $\u$ and is controlled by $\mathcal{A}_\mathcal{Z}$. The proof is provided in Appendix \ref{app:prelim} for completeness.

\begin{lemma}\label{lem:black box reduction}
Let $\regret_T^\mathcal{V}(\|\u\|) = \sumT (v_t - \|\u\|) \langle \z_t , \gt \rangle $ be the regret for learning $\|\u\|$ by Algorithm $\algv$ and let $\regret_T^\mathcal{Z}\left(\frac{\u}{\|\u\|}\right) = \sumT \langle \z_t - \frac{\u}{\|\u\|}, \gt \rangle$ be the regret for learning $\frac{\u}{\|\u\|}$ by $\mathcal{A}_\mathcal{Z}$. Then Algorithm \ref{alg:black box full} satisfies 
\begin{equation}\label{eq:blackboxfull}
    \regret_T(\u) = \regret_T^\mathcal{V}(\|\u\|) + \|\u\|\regret_T^\mathcal{Z}\left(\frac{\u}{\|\u\|}\right).
\end{equation}
\end{lemma}

\citet{cutkosky2018black} provide an algorithm to ensure $\regret_T^\mathcal{V}(\|\u\|) = \otil\left(1 + \|\u\| L\sqrt{T}\right)$, given that $\|\g_t\|_\star \leq L$. This algorithm satisfies the requirements described later in Interface \ref{alg:scalealg}, and will be used throughout this paper.

%% file: linear.tex

\section{Comparator-Adaptive Linear Bandits}\label{sec:linban}
Now, we apply the reduction of section \ref{sec:blackboxfull} to develop comparator-adaptive algorithms for linear bandits. We will see that in the unconstrained case, the reduction works almost without modification, but in the constrained case we will need to be more careful to enforce the constraints.

\subsection{Unconstrained Linear Bandits}\label{sec:unconlinban}
\begin{algorithm}[t]
\caption{Black-Box Reduction for Linear Bandits}\label{alg:black box bandit}
\begin{algorithmic}[1]
\STATE \textbf{Input:} Constrained Linear Bandit Algorithm $\mathcal{A}_\mathcal{Z}$ and unconstrained 1-d Algorithm $\mathcal{A}_\mathcal{V}$
\FOR{$t = 1 \ldots T$}
\STATE Get $\z_t \in \mathcal{Z}$ from $\mathcal{A}_\mathcal{Z}$
\STATE Get $v_t \in \reals$ from $\mathcal{A}_\mathcal{V}$
\STATE Play $\w_t = v_t \z_t$
\STATE Receive loss $\langle \w_t, \g_t \rangle$
\STATE Compute $\mathcal{L}_t=\frac{1}{v_t}\langle \w_t, \gt \rangle = \langle \z_t, \gt \rangle$.
\STATE Send $\mathcal{L}_t$ to Algorithm $\mathcal{A}_\mathcal{Z}$ as $t$-th loss value.
\STATE Send $\mathcal{L}_t$ to Algorithm $\mathcal{A}_\mathcal{V}$ as $t$-th loss value.
\ENDFOR
\end{algorithmic}
\end{algorithm}

We begin by discussing the unconstrained linear bandit setting, which turns out to be the easiest setting we consider. Following Algorithm \ref{alg:black box full}, we will still play $\w_t = v_t \z_t$. However, instead of taking a fixed $\z_t$ from a full-information algorithm, we take a random $\z_t$ from a \emph{bandit} algorithm. Importantly, we can recover $\langle \z_t, \gt \rangle$ exactly since $\inner{\w_t}{\g_t}\tfrac{1}{v_t} = \langle \z_t, \gt \rangle$. This means that we have enough information to send appropriate feedback to both $\algv$ and $\algz$ and apply the argument of Lemma \ref{lem:black box reduction}. Interestingly, we use a full-information one-dimensional algorithm for $\algv$, and only need $\algz$ to take bandit input. This is because $\algv$ gets full information in the form of $\langle \z_t, \g_t \rangle$.

The algorithm $\algz$ for learning the direction, on the other hand, now must be a bandit algorithm because intuitively we do not immediately get the full direction information $\g_t$ from the value of the loss alone. We will need this algorithm to fulfill the requirements described by Interface \ref{alg:directionlinban}. One such algorithm is given by continuous Exponential Weights on a constrained set (see \citet[section 6]{hoeven2018many} for details).

Our unconstrained linear bandit algorithm then is constructed from Algorithm \ref{alg:black box bandit} by choosing an algorithm that implements Interface \ref{alg:directionlinban} as $\algz$ and Interface \ref{alg:scalealg} with $l=\reals$ as $\algv$. Plugging in the guarantees of the individual algorithms and taking the expectation of \eqref{eq:blackboxfull}, the total expected regret is $\otil(1 + \|\u\|dL\sqrt{T})$. 
Compared to the full information setting we have gained a factor $d$ in the regret bound, which is unavoidable given the bandit feedback~\citep{dani2008price}. The formal result is below.

\begin{theorem}\label{th:uncon linear}
Suppose $\algz$ implements Interface~\ref{alg:directionlinban} with domain $\doma{Z}=\doma{B}$ and $\algv$ implements Interface~\ref{alg:scalealg} with $l = \reals_+$. Then Algorithm~\ref{alg:black box bandit} satisfies for all $\u\in \reals^d$:
\begin{equation*}
\E[\regret(\u)] = \otil(1 + \|\u\|dL\sqrt{T}).
\end{equation*}
\end{theorem}

\subsection{Constrained Linear Bandits}\label{sec:conlinban}
\begin{interface}[t]
\caption{Scale Learning Interface (see example implementation in \cite{cutkosky2018black})}\label{alg:scalealg}
\begin{algorithmic}[1]
\STATE \textbf{Input:} A line segment $l \subseteq \reals$
\FOR{$t = 1 \ldots T$}
\STATE Play $v_t \in l$
\STATE Receive loss value $g_t$ such that $|g_t| \leq L_\mathcal{V}$
\ENDFOR
\STATE \textbf{Ensure:}  for all $\vhat \in l$,
$\sumT (v_t - \vhat)g_t = \otil\left(1 + |\vhat|L_\mathcal{V}\sqrt{T}\right)$
\end{algorithmic}
\end{interface}

\begin{interface}[t]
\caption{Direction Learning Interface for Linear Bandits (see example implementation in \cite{hoeven2018many})}\label{alg:directionlinban}
\begin{algorithmic}[1]
\STATE \textbf{Input:} Domain $\doma{Z}$
\FOR{$t = 1 \ldots T$}
\STATE Play $z_t \in \doma{Z}$
\STATE Receive loss value $\inner{\z_t}{\g_t}$ such that $|\inner{\z_t}{\g_t}| \leq L$
\ENDFOR
\STATE \textbf{Ensure:} for all $\u \in \doma{Z}$,
$\E\left[\sumT \inner{\z_t - \u}{\g_t}\right] = \otil\left(d L \sqrt{T}\right)$
\end{algorithmic}
\end{interface}

The algorithm in the previous section only works for $\domainw = \reals^d$. In this section, we consider a compact set $\domainw \subset \reals^d$.   

In the full-information setting, \citet{cutkosky2018black} provide a projection technique for producing constrained algorithms from unconstrained ones. Unfortunately, this technique does not translate directly to the bandit setting, and we must be more careful in designing our constrained linear bandit algorithm. The key idea is to constrain the internal scaling algorithm $\algv$, rather than attempting to constrain the final predictions $\w_t$. Enforcing constraints on the scaling algorithm's outputs $v_t$ will naturally translate into a constraint on the final predictions $\w_t$.  



To produce a constrained linear bandit algorithm, we again use Algorithm~\ref{alg:black box bandit}, but now we instantiate $\algv$ implementating Interface~\ref{alg:scalealg} with $l = [0, 1]$ rather than $l=\reals_+$, and instantiate $\algz$ implementing Interface~\ref{alg:directionlinban} with $\doma{Z} = \domainw$ rather than $\doma{Z}=\doma{B}$. As in the unconstrained setting, this allows us to feed full information feedback to $\algv$, while at the same time now also guarantees that $\w_t \in \domainw$. The regret bound of this algorithm is given in Theorem \ref{th:constrained linear bandits}. The proof follows from combining Lemma \ref{lem:black box reduction} with the guarantees of Interfaces \ref{alg:scalealg} and \ref{alg:directionlinban} and can be found in Appendix \ref{app:conlinban}.

\begin{theorem}\label{th:constrained linear bandits}
Suppose $\algz$ implements \ref{alg:directionlinban} with domain $\doma{Z} = \doma{W}$ and $\algv$ implements \ref{alg:scalealg} with $l = [0, 1]$. Then Algorithm \ref{alg:black box bandit} satisfies for all $\u \in \doma{W}$,
\begin{equation*}
    \E[\regret_T(\u)] = \otil\left(1 + \|\u\|cdL\sqrt{T}\right).
\end{equation*}
\end{theorem}


If $\domainw$ is a unit ball, then $c = 1$. For other shapes of $\domainw$, recall that $c$ is at most $d$, which leads to a regret bound of $O\left(1 + \|\u\| d^2 L \sqrt{ T}\right)$.

%% file: convexdp1.tex

\section{Comparator-Adaptive Convex Bandits}\label{sec:convex bandits}
In the general convex bandit problem, it is not clear how to use the single evaluation point feedback $\ell_t(\w_t)$ to derive any useful information about $\ell_t$. Fortunately, \citet{flaxman2005online} solved this problem by using randomness to extract the gradients of a smoothed version of $\ell_t$. To adapt to the norm of the comparator, we employ the following tweaked version of smoothing used by \citet{flaxman2005online}:
\begin{equation}\label{eq:perturbed loss}
    \ell_t^v(\w) = \E_{\b \sim U(\ball)}[\ell_t(\w + v \delta \b)],
\end{equation}
where $v, \delta > 0$. In contrast to prior work using this framework, our smoothing now depends on the scaling parameter $v$.  Lemma~\ref{lem:estimate unbiased} gives the gradient of $\ell_t^v(\w)$ and is a straightforward adaptation of Lemma~2.1 by \citet{flaxman2005online}.
\begin{lemma}\label{lem:estimate unbiased}
For $\delta \in (0, 1]$, $v > 0$:
\begin{equation}\label{eq:gradient estimate}
    \nabla \ell_t^{v}(\w) = \frac{d}{v\delta} \E_{\s \sim U(\sphere)}[\ell_t(\w + v\delta \s)\s]. 
\end{equation}
\end{lemma}

With this lemma, we can estimate the gradient of the smoothed version of $\ell_t$ by evaluating $\ell_t$ at a random point, essentially converting the convex problem to a linear problem, except that one also needs to control the bias introduced by smoothing.
Note that this estimate scales with $\frac{1}{v}$, which can be problematic if $v$ is small. To deal with this issue, we require one extra assumption: the value of $\ell_t(\0)$ is known to the learner. 
As discussed in section \ref{sec:introduction}, this assumption holds for several applications, including some control or reinforcement learning problems, where $\0$ represents a nominal action with a known outcome. Furthermore, certain loss functions satisfy the second assumption by default, such as linear loss, logistic loss, and hinge loss. Without loss of generality we assume that $\ell_t(\0) = 0$, as we can always shift $\ell_t$ without changing the regret. 

Our general algorithm template is provided in Algorithm \ref{alg:convex bandit}. It incorporates the ideas of Algorithm \ref{alg:black box bandit}, but adds new smoothing and regularization elements in order to deal with the present more general situation.
More specifically, it again makes use of subroutine $\algv$, which learns the scaling. The direction is learned by Online Gradient Descent \citep{zinkevich2003}, as was also done by \citet{flaxman2005online}.
Given $\z_t$ and $v_t$, our algorithm plays the point $\w_t = v_t(\z_t+\delta\s_t)$ for some parameter $\delta$ and $s_t$ uniformly at random drawn from $\sphere$.
By equation (\ref{eq:gradient estimate}), we have
\begin{align}
\E\left[\frac{d}{v_t\delta}\ell_t(\w_t) s_t\right]&=\nabla \ell_t^{v_t}(v_t \z_t) \label{eqn:gradient_estimator}.
\end{align}
This means that we can use $\ghatt=\tfrac{d}{v_t\delta}\ell_t(\w_t) s_t$ as an approximate gradient estimate,
and we send this $\ghatt$ to to Online Gradient Descent as the feedback.
In other words, Online Gradient Descent itself is essentially dealing with a full-information problem with gradient feedback and is required to ensure a regret bound $\E[\sumT \inner{\z_t - \u}{\ghatt}] = \otil(\frac{dL}{\delta}\sqrt{T})$ for all $\u$ in some domain $\doma{Z}$. For technical reasons, we will also need to enforce $\z_t \in (1-\alpha)\doma{Z}$ for some $\alpha\in[0,1]$. This restriction will be necessary in the constrained setting to ensure  $v_t(\z_t + \delta \s_t) \in \doma{W}$ .

Next, to specify the feedback to the scaling learning black-box $\algv$, we define a surrogate loss function $\blosst(v)$ which contains a linear term $v \inner{\z_t}{\ghatt}$ and also a regularization term (see Algorithm~\ref{alg:convex bandit} for the exact definition).
The feedback to $\algv$ is then $\partial \blosst(v_t)$. 
Therefore, $\algv$ is essentially learning these surrogate losses, also with full gradient information.
The regularization term is added to deal with the bias introduced by smoothing. This term does not appear in prior work on convex bandits, and it is one of the key components needed to ensure that the final regret is in terms of the unknown $\|\u\|$.

\begin{algorithm}[t]
\caption{Black-Box Comparator-Adaptive Convex Bandit Algorithm}\label{alg:convex bandit}
\begin{algorithmic}[1]
\STATE \textbf{Input:} Scaling algorithm $\algv$, $\delta \in (0, 1]$, $\alpha \in[0, 1]$, domain $\doma{Z} \subseteq \ball$, and learning rate $\eta$
\STATE Set $\z_1 = \0$
\FOR{$t = 1 \ldots T$}
\STATE Get $v_t$ from $\algv$
\STATE Sample $\s_t \sim U(\sphere)$
\STATE Set $\w_t = v_t (\z_t + \delta \s_t)$
\STATE Play $\w_{t}$
\STATE Receive $\ell_t(\w_{t})$
\STATE Set $\ghatt = \frac{d}{v_t\delta}\ell_t(\w_{t})s_t$
\IF{$\ell_t$ is $\beta$-smooth}
\STATE Set $\blosst(v) = v \inner{\z_t}{\ghatt}  + \beta \delta^2 v^2 $
\ELSE
\STATE Set $\blosst(v) = v \inner{\z_t}{\ghatt}  + 2 \delta L |v| $
\ENDIF
\STATE Send $\partial \blosst(v_t)$ to algorithm $\mathcal{A}_\mathcal{V}$ as the $t$-th loss value 
\STATE Update $\z_{t+1} = \argmin_{\z \in (1-\alpha)\doma{Z}} \eta \inner{\z}{\ghatt} + \|\z_t - \z\|_2^2$
\ENDFOR
\end{algorithmic}
\end{algorithm}


Algorithm \ref{alg:convex bandit} should be seen as the analogue of the black-box reduction of Algorithm \ref{alg:black box full}, but for bandit feedback instead of full information. 
The expected regret guarantee of Algorithm \ref{alg:convex bandit} is shown below, and the proof can be found in appendix \ref{app:convex bandits}.

\begin{lemma}\label{lem:convex bandits}
Suppose $\algv$ implements Interface \ref{alg:scalealg} with $l \subseteq \reals_+$. Suppose $\w_t\in \doma{W}$ for all $t$, and let $L_{\doma{V}}=\max_t \partial \blosst(v_t)$. Then Algorithm \ref{alg:convex bandit} with $\delta, \alpha \in (0, 1]$ and $\eta = \sqrt{\frac{\delta^2}{4(dL)^2 T}}$ satisfies for all $\|\u\| \in l$ and $r > 0$ with $\frac{\u r}{\|\u\|} \in \doma{Z}$, 
\begin{equation*}
\begin{split}
    \E&\left[\regret_T(\u)\right] =  \otil\left(1 + T  \delta L\frac{\|\u\|}{r} + \frac{\|\u\|}{r} L_\doma{V} \sqrt{T} + \frac{\|\u\|dL}{r\delta} \sqrt{T} + \alpha \|\u\|_2 T L \right).
\end{split}
\end{equation*}
In addition, if $\ell_t$ is also $\beta$-smooth for all $t$, then 
we have
\begin{equation*}
\begin{split}
    \E& \left[\regret_T(\u)\right] = \otil\left(1+ T \beta \delta^2 \left(\frac{\|\u\|}{r}\right)^2 + \frac{\|\u\|}{r} L_\doma{V} \sqrt{T} + \frac{\|\u\|}{r}  \frac{dL}{\delta} \sqrt{T} + \alpha \|\u\|_2 T L\right).
\end{split}
\end{equation*}
\end{lemma}

This bound has two main points not obviously under our direct control: the assumption that the $\w_t$ lie in $\doma{W}$, and the value of $L_{\doma{V}}$, which is a bound on $|\partial\blosst(v_t)|$. In the remainder of this section we will specify the various settings of Algorithm \ref{alg:convex bandit} that guarantee that $w_t\in \doma{W}$ and that $L_{\doma{V}}$ is suitably bounded: two setting for the unconstrained setting and one for the constrained setting. The $\alpha \|\u\|T L$ term due to $\z_t \in (1-\alpha)\doma{Z}$ rather than $\z_t \in \doma{Z}$, which induces a small amount of bias.
The $r$ in Lemma \ref{lem:convex bandits} is to ensure that we satisfy the requirements for Online Gradient Descent to have a suitable regret bound. For unconstrained convex bandits $r = 1$. For constrained convex bandits we will find that $\frac{1}{r} = c$ (recall that we assume that $\frac{1}{c}\ball \subseteq \doma{W} \subseteq \ball$).

\subsection{Unconstrained Convex Bandits}\label{sec:unconstrained convex bandits}

In this section we instantiate Algorithm~\ref{alg:convex bandit} and derive regret bounds for either general convex losses or convex and smooth losses. We start with general convex losses. Since $\doma{W} = \reals^d$, we do not need to ensure that $\z_t + \delta\s_t \in \doma{W}$ and we can safely set $\alpha = 0$. This choice guarantees that $\z_t + \delta \s_t \in 2 \ball$ and that $|\partial \blosst(v_t)| \leq \frac{2dL}{\delta} + 2\delta L$. Then, Lemma \ref{lem:convex bandits} directly leads to Theorem \ref{th:uncon lipschitz} (the proof is deferred to appendix \ref{app:unconstrained convex bandits}).

\begin{theorem}\label{th:uncon lipschitz}
Supppose $\algv$ implements Interface \ref{alg:scalealg} with $l = \reals_+$. Then Algorithm \ref{alg:convex bandit} with $\delta = \min\{1, \sqrt{d}T^{-\frac{1}{4}}\}$, $\doma{Z} = \ball$, $\alpha = 0$, and $\eta = \sqrt{\frac{\delta^2}{4(dL)^2 T}}$ satisfies for all $\u \in \reals^d$,
\begin{equation*}
\begin{split}
    \E\left[\regret_T(\u)\right]
    =  \otil\left(1 +  \|\u\|Ld\sqrt{T} + \|\u\|L \sqrt{d}T^\frac{3}{4}\right).
\end{split}
\end{equation*}
\end{theorem}

For unconstrained smooth bandits, we face an extra challenge. To bound the regret of Algorithm \ref{alg:convex bandit}, $|\partial \blosst(v_t)| = |\inner{\z_t}{\ghatt} + \beta 2 \delta^2 v_t|$ must be bounded. Now in contrast to the linear or Lipschitz cases, in the smooth case $\blosst(v_t)$ is not Lipschitz over $\reals_+$. We will address this by artificially constraining $v_t$. Specifically, we ensure that $v_t \leq \frac{1}{\delta^3}$, which implies $|\delta^2 v_t| = O\left(\frac{1}{\delta}\right)$. This makes the Lipschitz constant of $\blosst$ to be dominated by the gradient estimate $\ghatt$ rather than the regularization. To see how this affects the regret bound, consider two cases, $\|\u\|_2 \leq \frac{1}{\delta^3}$ and $\|\u\|_2 > \frac{1}{\delta^3}$. If $\|\u\|_2 \leq \frac{1}{\delta^3}$ then we have not hurt anything by constraining $v_t$ since $\|\u\|_2$ satisfies the same constraint. If instead $\|\u\|_2 > \frac{1}{\delta^3}$ then the consequences for the regret bound are not immediately clear. However, following a similar technique in \cite{cutkosky2019artificial}, we utilize the fact that the regret against $\0$ is $O(1)$ and the Lipschitz assumption to show that we have added a penalty of only $O(\|\u\|_2LT)$: 
\begin{align*}
\E[\regret_T(\u)] = \E[\regret_T(\0)] + \sumT \E[\ell_t(\0) - \ell_t(\u)] = O(1 + \|\u\|_2LT).
\end{align*} 
Since $\|\u\|_2 > \frac{1}{\delta^3}$ the penalty for constraining $v_t$ is $O(\|\u\|_2 LT) = O(\|\u\|^2_2 L\delta^3 T)$, which is $O(\|\u\|^2_2 L \sqrt{T})$ if we set $\delta = O(T^{-1/6})$. The formal result can be found below and its proof can be found in appendix \ref{app:unconstrained convex bandits}.

\begin{theorem}\label{th:unconsmooth}
Suppose $\algv$ implements Interface \ref{alg:scalealg} with $l = \reals_+$ and that $\ell_t$ is $\beta$-smooth for all $t$. Then Algorithm \ref{alg:convex bandit} with $\delta = \min\{1, (dL)^{1/3} T^{-1/6}\}$, $\doma{Z} = \ball$, $\alpha = 0$, and $\eta = \sqrt{\frac{\delta^2}{4(dL)^2 T}}$ satisfies for all $\u \in \reals^d$,
\begin{equation*}
\begin{split}
\E \left[\regret_T(\u) \right]  = \otil\left(1 +  \max\{\|\u\|^2, \|\u\|\} \beta (dLT)^\frac{2}{3} + \max\{\|\u\|^2_2, \|\u\|\} dL^2\beta\sqrt{T}\right).
\end{split}
\end{equation*}
\end{theorem}


\subsection{Constrained convex bandits}\label{sec:con convex bandits}

For the constrained setting we will set $\doma{Z} = \doma{W}$ and $\alpha = \delta$. This ensures that $v_t(\z_t + \delta\s_t) \in \doma{W}$ and we can apply Lemma \ref{lem:convex bandits} to find the regret bound in Theorem \ref{th:con convex bandits} below. Compared to the unconstrained setting, the regret bound now scales with $c$, which is due to reshaping trick discussed in \cite{flaxman2005online}.

\begin{theorem}\label{th:con convex bandits}
Suppose $\algv$ implements Interface \ref{alg:scalealg} with $l = (0, 1]$. 
Then Algorithm \ref{alg:convex bandit} with $\delta = \min\{1, \sqrt{d} T^{-1/4}\}$, $\doma{Z} = \doma{W}$, $\alpha = \delta = \min\{1, \sqrt{d}T^{-1/4}\}$, and $\eta = \sqrt{\frac{\delta^2}{4(dL)^2 T}}$ satisfies for all $\u \in \domainw$,
\begin{equation*}
\begin{split}
\E  \left[\regret_T(\u) \right] = \otil\left(1 + (\|\u\|_2 + c\|\u\|)\sqrt{d}T^{3/4} + c\|\u\| dL\sqrt{T}\right).
\end{split}
\end{equation*}
\end{theorem}


%% file: conclusion.tex

\section{Conclusion}\label{sec:conclusion}
In this paper, we develop the first algorithms that have comparator-adaptive regret bounds for various bandit convex optimization problems. The regret bounds of our algorithms scale with $\|\u\|$, which may yield smaller regret in favourable settings. 

For future research, there are a number of interesting open questions. First, our current results do not encompass improved rates for smooth losses on constrained domains. At first blush, one might feel this is relatively straightforward via methods based on self-concordance \cite{saha2011improved}, but it turns out that while such techniques provide good direction-learning algorithms, they may cause the gradients provided to the \emph{scaling} algorithm to blow-up. Secondly, there is an important class of loss functions for which we did not obtain norm adaptive regret bounds: smooth and strongly convex losses. It is known that in this case an expected regret bound of $O(d\sqrt{T})$ can be efficiently achieved~\citep{hazan2014bandit}. However, to achieve this regret bound the algorithm of \citet{hazan2014bandit} uses a clever exploration scheme, which unfortunately leads to sub-optimal regret bounds for our algorithms. 



%% file: appendixdp1.tex

\section{Details from section \ref{sec:prelim}}\label{app:prelim}

\begin{proof}[Proof of Lemma \ref{lem:black box reduction}]
By definition we have
\begin{equation*}
\begin{split}
\regret_T(\u) =  \sumT \langle \w_t - \u, \gt \rangle = & \sumT \langle \z_t, \gt \rangle (v_t - \|\u\|) + \|\u\|\sumT\langle \z_t - \frac{\u}{\|\u\|}, \gt \rangle \\
= & \regret_T^\mathcal{V}(\|\u\|) + \|\u\| \regret_T^\mathcal{Z}\left(\frac{\u}{\|\u\|}\right).
\end{split}
\end{equation*}
\end{proof}

\section{Details from section \ref{sec:linban}}\label{app:conlinban}

\begin{proof}[Proof of Theorem \ref{th:constrained linear bandits}] 
For any fixed $\u \in \domainw$, let $r = \max_{\frac{r'\u}{\|\u\|} \in \domainw} r'$.
Note that by definition we have $\frac{\|\u\|}{r} \in [0,1]$ and $\frac{r\u}{\|\u\|}\in \domainw$.
Therefore, similar to the proof of Lemma~\ref{lem:black box reduction}, we decompose the regret against $\u$ as:
\begin{equation*}
\begin{split}
\regret_T(\u) = \sumT \langle \w_t - \u, \gt \rangle  = \sumT \langle \z_t, \gt \rangle \left(v_t - \frac{\|\u\|}{r}\right) + \frac{\|\u\|}{r}\sumT\langle \z_t - \frac{r\u}{\|\u\|}, \gt \rangle,
\end{split}
\end{equation*}
which, by the guarantees of $\algv$ and $\algz$,\footnote{Note that the condition $|\langle z_t, g_t \rangle|\leq 1$ in Algorithm~\ref{alg:directionlinban} indeed holds in this case since $\doma{Z} = \doma{W} \subseteq \ball$ and $\|g_t\|_2\leq L$ by the Lipschitzness condition.} is bounded in expectation by
\[
\otil\left(\frac{\|\u\|}{r}L\sqrt{T} + \frac{\|\u\|}{r}dL\sqrt{T}\right).
\]
Finally noticing $\frac{1}{c} \leq r$ by the definition of $c$ finishes the proof.
\end{proof}

\section{Details from section \ref{sec:convex bandits}}\label{app:convex bandits}

\begin{proof}[Proof of Lemma \ref{lem:convex bandits}]
Denote by $\wtilt = v_t\z_t$. By Jensen's inequality we have
\begin{equation}\label{eq:convex jensen}
\begin{split}
   \sumT  \E\left[\ell_t(\w_t) - \ell_t(\u)\right] = & \E\left[\sumT \edvt(\w_t) - \edt(\ud)\right] + \sumT \E\left[\edt(\w_t) - \edvt(\w_t)\right] \\
    \leq & \sumT \E\left[\edvt(\w_t) - \edt(\ud)\right].
\end{split}
\end{equation}
We now continue under the assumption that $\ell_t$ is $L$-Lipschitz. After completing the proof of the first equation of Lemma \ref{lem:convex bandits} we use the $\beta$-smoothness assumption to prove the second equation of Lemma \ref{lem:convex bandits}.

Using the $L$-Lipschitz assumption we proceed:
\begin{equation*}
\begin{split}
 \sumT \E\left[\edvt(\w_t) - \edt(\ud)\right] \leq &\sumT \E\left[\edvt(\w_t) - \edvt(\ud)\right]  + \sumT \E\left[\edvt(\ud) - \edt(\ud)\right] \\
\leq & \sumT \E\left[\edvt(\w_t) - \edvt(\ud)\right] + \E[L |v_t|\|\delta \s_t\|_2 ]  \\
\leq & \sumT \E\left[\edvt(\w_t) - \edvt(\ud)\right] + \E[\delta L |v_t|]  \\
    = & \sumT \E\left[\edvt(\wtilt) - \edvt(\ud)\right] + \E[\delta L |v_t| ]  \\
    & + \sumT \E\left[\edvt(\w_t) - \edvt(\wtilt)\right] \\
     \leq & \sumT \E\left[\edvt(\wtilt) - \edvt(\ud)\right] +  2\E[\delta L |v_t| ].
\end{split}
\end{equation*}
Now, by using the $L$-Lipschitz assumption once more we find that 
\begin{equation}\label{eq:1-alpha}
\sumT \E[\edvt((1-\alpha)\ud) - \edvt(\ud)] \leq \alpha \|\u\|_2 T L
\end{equation}

By using equation \eqref{eq:1-alpha}, the convexity of $\edvt$, and Lemma \ref{lem:estimate unbiased} we continue with:
\begin{equation*}
\begin{split}
 \sumT  \E\left[\ell_t(\w_t) - \ell_t(\u)\right] \leq & \sumT \E\left[\inner{\wtilt - (1-\alpha)\ud}{\ghatt}\right] + 2\E[\delta L |v_t| ] + \alpha \|\u\|_2T L  \\
  = & \sumT \E\left[\left(v_t - \frac{\|\u\|}{r}\right)\inner{\z_t}{\ghatt}\right] + \E\left[\frac{\|\u\|}{r} \inner{\z_t - \tud}{\ghatt}\right] \\
  & + \sumT 2\E[\delta L |v_t| ]  + \alpha \|\u\|_2T L\\
    = & \sumT \E\left[\blosst(v_t) - \blosst\left(\frac{\|\u\|}{r}\right)\right] + \sumT \frac{\|\u\|}{r} \E\left[\inner{\z_t - \tud}{\ghatt}\right] \\
    & + 2 T \delta L  \frac{\|\u\|}{r} + \alpha \|\u\|_2T L
\end{split}
\end{equation*}
where $\blosst(v) = v \inner{\z_t}{\ghatt}  + 2\delta L |v|$ as defined in Algorithm~\ref{alg:convex bandit}, $\tud = \frac{r}{\|\u\|} (1-\alpha)\u$, and $r > 0$ is such that $\frac{\u r}{\|\u\|} \in \doma{Z}$.

Finally, by using the convexity of $\blosst$, plugging in the guarantee of $\algv$, and using Theorem \ref{th:regret GD} we conclude the proof of the first equation of Lemma \ref{lem:convex bandits}: 
\begin{equation*}
\begin{split}
 & \sumT \E\left[\ell_t(\w_t) - \ell_t(\u)\right] \\
 & \leq   2 T  \delta L \frac{\|\u\|}{r} + \E\left[\sumT \left(v_t -\frac{\|\u\|}{r}\right) \partial \blosst(v_t)\right] + \frac{\|\u\|}{r}\E\left[\sumT \langle \z_t - \tud, \ghatt \rangle\right] + \alpha \|\u\|_2T L\\
    & =\otil\left(1 + T  \delta L\frac{\|\u\|}{r} + \frac{\|\u\|}{r} L_\doma{V} \sqrt{T} + \frac{\|\u\|dL}{r\delta} \sqrt{T} + \alpha \|\u\|_2T L \right).
\end{split}
\end{equation*}

Next, we continue from equation \eqref{eq:convex jensen} under the smoothness condition. Using the definition of smoothness we find
\begin{equation*}
\begin{split} 
\sumT  \E\left[\edvt(\w_t) - \edt(\ud)\right] \leq & \sumT \E\left[\edvt(\w_t) - \edvt(\ud)\right] + \sumT \E\left[\edvt(\ud) - \edt(\ud)\right] \\
    \leq & \sumT \E\left[\edvt(\w_t) - \edvt(\ud)\right] + \E\left[\half \beta |v_t|^2\|\delta \s_t\|_2^2 \right]  \\
    = & \sumT \E\left[\edvt(\w_t) - \edvt(\ud)\right] + \E\left[\half \delta^2 |v_t|^2 \beta \right]  \\
    = & \sumT \E\left[\edvt(\wtilt) - \edvt(\ud)\right] + \E\left[\half \delta^2 |v_t|^2 \beta \right] \\
    & +\sumT \E\left[\edvt(\w_t) - \edvt(\wtilt)\right] \\
    \leq & \sumT \E\left[\edvt(\wtilt) - \edvt(\ud)\right] + \E\left[\beta  \delta^2 |v_t|^2\right].   
\end{split}
\end{equation*}
Using equation \eqref{eq:1-alpha}, the convexity of $\edvt$, and Lemma \ref{lem:estimate unbiased} we continue with:
\begin{equation*}
\begin{split}
& \sumT  \E\left[\ell_t(\w_t) - \ell_t(\ud)\right]   \\
& \leq\sumT \E\left[\inner{\wtilt - (1-\alpha) \ud}{\ghatt}\right] + \E\left[\beta \delta^2|v_t|^2  \right] + \alpha \|\u\|_2 T L\\
    &=  \sumT \E\left[ \left(v_t - \frac{\|\u\|}{r}\right) \inner{\z_t}{\ghatt}\right] + \E\left[\beta  \delta^2|v_t|^2 \right] + \sumT \frac{\|\u\|}{r} \E\left[\inner{\z_t - \tud}{\ghatt}\right]  + \alpha \|\u\|_2T L\\
   & =  T \beta \delta^2 \left(\frac{\|\u\|}{r}\right)^2    + \sumT \E\left[\blosst(v_t) - \blosst\left(\frac{\|\u\|}{r}\right)\right] + \sumT \frac{\|\u\|}{r} \E\left[\inner{\z_t - \tud}{\ghatt}\right] + \alpha \|\u\|_2T L,
\end{split}
\end{equation*}
where $\blosst(v) = v \inner{\z_t}{\ghatt}  +  \beta \delta^2 v^2$ as defined in Algorithm~\ref{alg:convex bandit}. Finally, by using the convexity of $\blosst$, plugging in the guarantee of $\algv$, and using Theorem \ref{th:regret GD} we  conclude the proof:
\begin{equation*}
\begin{split}
& \sumT  \E\left[\ell_t(\w_t) - \ell_t(\u)\right]  \\
& \leq T \beta \delta^2 \left(\frac{\|\u\|}{r}\right)^2 + \E\left[\sumT \left(v_t -\frac{\|\u\|}{r}\right) \partial \blosst(v_t)\right]  + \frac{\|\u\|}{r}\E\left[\sumT \langle \z_t - \tud, \ghatt \rangle\right] + \alpha \|\u\|_2T L \\
     & = \otil\left(1 + T \beta \delta^2 \left(\frac{\|\u\|}{r}\right)^2 + \frac{\|\u\|}{r} L_\doma{V} \sqrt{T} + \frac{\|\u\|}{r}  \frac{dL}{\delta} \sqrt{T} + \alpha \|\u\|_2T L\right).
\end{split}
\end{equation*}

\end{proof}

\begin{theorem}\label{th:regret GD}
Suppose that $\ell_t(\0) = 0$, that $\ell_t$ is $L$-Lipschitz for all $t$, and that $\doma{Z} \subseteq \ball$. For $\u \in (1-\alpha)\doma{Z}$, Online Gradient Descent on $(1-\alpha)\doma{Z}$ with learning rate $\eta = \sqrt{\frac{\delta^2}{(dL)^2 4 T}}$ satisfies
\begin{equation*}
\begin{split}
\E\left[\sumT \inner{\z_t - \u}{\ghatt}\right] 
\leq &  2\frac{dL}{\delta} \sqrt{T}.
\end{split}
\end{equation*}
\end{theorem}
\begin{proof}
The proof essentially follows from the work of \citet{zinkevich2003, flaxman2005online} and using the assumptions that $\ell_t(\0) = 0$ and that $\ell_t$ is $L$-Lipschitz. We start by bounding the norm of the gradient estimate:
\begin{equation} \label{eq:bounded gradient}
\begin{split}
\|\ghatt\|_2 = & \frac{d}{v_t\delta}|\ell_t(\w_{t})|\|\s_t\|_2 \\
= & \frac{d}{v_t\delta}|\ell_t(v_t(\z_{t}+\delta \s_t)) - \ell_t(\0)| \\
\leq & \frac{dL\|\z_{t}+\delta \s_t\|_2}{\delta} \leq \frac{dL(1-\alpha + \delta)}{\delta}
\end{split}
\end{equation}
By using equation \eqref{eq:bounded gradient} and the regret bound of Online Gradient Descent \citep{zinkevich2003} we find that 
\begin{align*}
\sumT \inner{\z_t}{\ghatt} - \min_{\z \in (1-\alpha)\doma{Z}} \sumT \inner{\z}{\ghatt} \leq & \frac{(1-\alpha)}{2\eta} + \frac{\eta}{2} \sumT \|\ghatt\|_2^2 \\
\leq & \frac{(1-\alpha)}{2\eta} + \frac{\eta}{2} \left(\frac{dL(1-\alpha + \delta)}{\delta}\right)^2 T \\
\leq & \frac{1}{2\eta} + 2\eta \left(\frac{dL}{\delta}\right)^2 T 
\end{align*}
Plugging in $\eta = \sqrt{\frac{\delta^2}{(dL)^2 4 T}}$ completes the proof.
\end{proof}

\subsection{Details of section \ref{sec:unconstrained convex bandits}}\label{app:unconstrained convex bandits}

\begin{proof}[Proof of Theorem \ref{th:uncon lipschitz}] 
First,  since $\ell_t(\0)=0$, $\ell_t$ is $L$-Lipschitz, and $\z_t \in (1-\alpha)\doma{Z} = (1-\alpha)\ball$ we have that 
\begin{equation}\label{eq:bounded inner product}
\begin{split}
\inner{\z_t}{\ghatt} \leq \|\z_t\|_2\|\ghatt\|_2 \leq (1-\alpha)\frac{dL(1-\alpha + \delta)}{\delta} \leq \frac{2dL}{\delta},
\end{split}
\end{equation}
where the first inequality is the Cauchy-Schwarz inequality and the second is due to equation \eqref{eq:bounded gradient}. Since $|\partial \blosst(v_t)| \leq |\inner{\z_t}{\ghatt}| + 2 \delta L  = L_{\mathcal{V}}$ we can use Lemma \ref{lem:convex bandits} to find
\begin{equation*}
\begin{split}
\E&\left[\regret_T(\u)\right]  =  \otil\left(\delta TL\|\u\| + \|\u\|\frac{dL}{ \delta}\sqrt{T} + \alpha TL \|\u\|_2 \right).
\end{split}
\end{equation*}
Plugging in $\alpha = 0$ and $\delta = \min\{1, \sqrt{d}T^{-\frac{1}{4}}\}$ completes the proof. 
\end{proof}

\begin{proof}[Proof of Theorem \ref{th:unconsmooth}] 
By equation \eqref{eq:bounded inner product} $|\inner{\z_t}{\ghatt}| \leq \frac{2dL}{\delta}$.
Since $v_t \leq \frac{1}{\delta^3}$ we have that 
\begin{equation*}
\begin{split}
|\partial \blosst(v_t)| \leq & \frac{dL}{\delta} + 2|v_t|\beta\delta^2 \leq \frac{dL + 2\beta}{\delta} 
\leq \frac{\beta (dL + 2)}{\delta} 
\end{split}
\end{equation*}
If $\|\u\|_2 \leq \frac{1}{\delta^3}$ applying Lemma \ref{lem:convex bandits} with $\alpha = 0$ gives us 
\begin{equation}\label{eq:consmoothboundedunorm}
\begin{split}
\E \left[\sumT \ell_t(\w_t) - \ell_t(\u) \right] = \otil\left(1 + T \beta \delta^2\|\u\|^2 + \|\u\|\frac{dL \beta}{\delta}\sqrt{T}\right).
\end{split}
\end{equation}

If $\|\u\|_2 > \frac{1}{\delta^3}$ then using the Lipschitz assumption on $\ell_t$ and equation  \eqref{eq:consmoothboundedunorm} with $\u = \0$ gives us
\begin{equation}\label{eq:consmoothnotboundedunorm}
\begin{split}
\E \left[\sumT \ell_t(\w_t) - \ell_t(\u) \right]  = & \E \left[\sumT \ell_t(\w_t) -  \ell_t(\0) + \ell_t(\0) - \ell_t(\u) \right]\\
& = \otil(1 + \|\u\|_2 LT)\\
& = \otil(1 + \|\u\|^2_2 \delta^3 LT),
\end{split}
\end{equation}

where we used that $\|\u\|_2 \geq \frac{1}{\delta^3}$. Adding equations \eqref{eq:consmoothboundedunorm} and \eqref{eq:consmoothnotboundedunorm}
gives
\begin{align*}
\E& \left[\sumT \ell_t(\w_t) - \ell_t(\u) \right] = \otil \bigg(1 + \|\u\|^2_2 \delta^3 LT + T \beta \delta^2\|\u\|^2  + \|\u\|\frac{\beta dL}{\delta}\sqrt{T}\bigg)
\end{align*}

Setting $\delta = \min\{1, (dL)^{1/3} T^{-1/6}\}$ gives us
\begin{align*}
 \E \left[\sumT \ell_t(\w_t) - \ell_t(\u) \right] = \otil\left(1 +  \max\{\|\u\|^2, \|\u\|\} \beta (dLT)^\frac{2}{3} + \max\{\|\u\|^2_2, \|\u\|\} dL^2\beta\sqrt{T}\right).
\end{align*}

\end{proof}

\subsection{Details of section \ref{sec:con convex bandits}}\label{app:con convex bandits}

\begin{proof}[Proof of Theorem \ref{th:con convex bandits}]
First, to see that $\z_t + \delta \s_t \in \doma{W}$ recall that by assumption $\doma{W} \subseteq \ball$. Since $\alpha = \delta$ we have that $\z_t + \delta \s_t \in (1-\alpha)\doma{W} + \delta \sphere \subseteq (1-\delta)\doma{W} + \delta \doma{W} = \doma{W}$. For any fixed $\u \in \domainw$, let $r = \max_{\frac{r'\u}{\|\u\|} \in \domainw} r'$. Note that by definition we have $\frac{\|\u\|}{r} \in [0,1]$ and $\frac{r\u}{\|\u\|} \in \domainw$. By using equation \eqref{eq:bounded inner product} we can see that $|\partial\blosst(v_t)| \leq \frac{dL}{\delta} + 2\delta L$. By definition, $\frac{1}{r} \leq c$. 
This implies that the regret of $\algv$ is $\otil\left(1 + \frac{\|\u\|}{r}\frac{dL}{\delta}\sqrt{T}\right)$.  Applying Lemma \ref{lem:convex bandits} with the parameters above we find
\begin{equation*}
\begin{split}
\E \left[\sumT \ell_t(\w_t) - \ell_t(\u) \right]  =  \otil\left(1 + (\|\u\|_2 + c\|\u\|) T L \delta + c\|\u\| \delta L \sqrt{T} + c\|\u\| \frac{dL}{\delta}\sqrt{T}\right).
\end{split}
\end{equation*}

Finally, setting $\delta = \min\{1, \sqrt{d} T^{-1/4}\}$ completes the proof:
\begin{equation*}
\begin{split}
\E& \left[\sumT \ell_t(\w_t) - \ell_t(\u) \right] =  \otil\left(1 + (\|\u\|_2 + c\|\u\|)\sqrt{d}T^{3/4} + c\|\u\| dL\sqrt{T}\right).
    \end{split}
\end{equation*}

%

\end{proof}